\documentclass{article}

\usepackage{graphicx} 
\usepackage{subfigure} 

\usepackage{natbib}

\usepackage{algorithm}
\usepackage{algorithmic}

\usepackage[margin=28mm]{geometry}
\usepackage{amsfonts}
\usepackage{amsmath}
\usepackage{url}
\usepackage{tikz}
\usetikzlibrary{arrows}
\usepackage{wrapfig}

\newdimen\arrowsize
\pgfarrowsdeclare{arcsq}{arcsq}
{
  \arrowsize=0.2pt
  \advance\arrowsize by .5\pgflinewidth
  \pgfarrowsleftextend{-4\arrowsize-.5\pgflinewidth}
  \pgfarrowsrightextend{.5\pgflinewidth}
}
{
  \arrowsize=1.5pt
  \advance\arrowsize by .5\pgflinewidth
  \pgfsetdash{}{0pt} 
  \pgfsetroundjoin   
  \pgfsetroundcap    
  \pgfpathmoveto{\pgfpoint{0\arrowsize}{0\arrowsize}}
  \pgfpatharc{-90}{-140}{4\arrowsize}
  \pgfusepathqstroke
  \pgfpathmoveto{\pgfpointorigin}
  \pgfpatharc{90}{140}{4\arrowsize}
  \pgfusepathqstroke
}

\newcommand{\independent}{\mbox{${}\perp\mkern-11mu\perp{}$}}
\newcommand{\notindependent}{\mbox{${}\not\!\perp\mkern-11mu\perp{}$}}

\newcommand{\prob}{{\mathbf P}}

\newcommand{\Z}{{\mathbb Z}}
\newcommand{\N}{{\mathbb N}}
\newcommand{\X}{{\mathbf X}}
\newcommand{\mean}{{\mathbf E}}  
  
\newcommand{\eps}{{N}}

\newcommand{\C}[1]{\mathcal{#1}}
\newcommand{\B}[1]{\mathbf{#1}}
\newcommand{\PA}[2][]{{\B{PA}}^{#1}_{#2}}
\newcommand{\tPA}[2][]{{\tilde{\B{PA}}}^{#1}_{#2}}

\newtheorem{definition}{Definition}

\newtheorem{Lemma}{Lemma}

\newtheorem{Theorem}{Theorem}

\newenvironment{proof}[1][. ]{{\bf Proof #1}}{\hfill$\square$\vskip\baselineskip}
\newenvironment{Definition*}{{\bf Definition}}{}

\begin{document} 

\title{Causal Inference on Time Series using Structural Equation Models}

\author{Jonas Peters\footnote{Max Planck Institute
for Intelligent Systems,  
T\"ubingen, Germany}
\footnote{Seminar for Statistics, ETH Zurich, Switzerland}\\
peters@stat.math.ethz.ch  
\and 
Dominik Janzing$^*$\\
janzing@tuebingen.mpg.de
\and 
Bernhard Sch\"olkopf$^*$\\
bs@tuebingen.mpg.de
}



\maketitle

\begin{abstract}
Causal inference uses observations to infer the causal structure of the data generating system. 
We study a class of functional models that we call Time Series Models with Independent Noise (TiMINo). These models require independent residual time series, whereas traditional methods like Granger causality exploit the variance of residuals.
There are two main contributions: (1) {\it Theoretical:} By restricting the model class (e.g. to additive noise) we can provide a more general identifiability result than existing ones. This result incorporates lagged and instantaneous effects that can be nonlinear and do not need to be faithful, and non-instantaneous feedbacks between the time series. (2) {\it Practical:} If there are no feedback loops between time series, we propose an algorithm based on non-linear independence tests of time series. When the data are causally insufficient, or the data generating process does not satisfy the model assumptions, this algorithm may still give partial results, but mostly avoids incorrect answers. An extension to (non-instantaneous) feedbacks is possible, but not discussed.
It outperforms existing methods on artificial and real data. Code can be provided upon request.
\end{abstract} 

\section{Introduction}
We consider finitely many time series $X^i_t, i \in V$, with a maximal time order of $p$, that is we assume no influence from $X^i_{t-k}$ on $X^j_{t}$ for $k>p$. We further assume stationarity: the influence from $X^i_{t-k}$ on $X^j_{t}$ is required to be the same for all $t$. The question whether $X^i$ is causing $X^j$ now reads as whether there is a causal influence from some $X^i_{t-k}$ on $X^j_{t}$, for $0 \leq k < p$. All models assume homoscedastic noise.

We first review causal inference on iid data, that is in the case with no time structure, in Section~\ref{sec:2}. Note that iid methods cannot be applied directly on time series data because a common history might introduce complicated dependencies between contemporaneous data $X_t$ and $Y_t$.
Motivated by the iid case, \citet{Chu2008} and \citet{Hyvarinen2008} propose approaches for the time series setting that include linear instantaneous effects. 
We describe these methods together with Granger causality in Section~\ref{sec:3}.
All of them encounter similar problems: none of them are general enough to include nonlinear instantaneous effects or hidden common causes. Furthermore, when the model assumptions are violated the methods give incorrect results and one draws false causal conclusions without noticing.
We propose 
to use time series models with independent noise ({\it TiMINo}) that include nonlinear and instantaneous effects. 
The model is based on Functional Models (also known as Structural Equation Models) and assumes $X_t$ to be a function of all direct causes 
and some noise variable, the collection of which is supposed to be jointly independent. 
This constitutes a relatively straight-forward extension on iid methods, but we regard the benefits in the setting of time series as substantial: 
In Section~\ref{sec_timino} we prove that for TiMINo models the full causal structure can be recovered from the distribution.
Section~\ref{sec_alg} introduces an algorithm ({\it TiMINo causality}) that recovers the model structure from a finite sample. It covers a broader class of models than existing methods and can be run with any provided algorithm for fitting time series. 
If the data do not satisfy the assumptions, TiMINo causality remains mostly (see Section~\ref{sec:wea}) undecided 
instead of drawing wrong causal conclusions.
The methods are applied to simulated and real data sets in Section~\ref{sec_exp}.


\section{Causal inference on iid data} \label{sec:2}
Inferring causal relations from observational data is challenging when 
interventions are not applicable. Given iid samples from 
$\prob^{(X^i),i\in V}$, we try to find the underlying causal structure of the 
variables $X^i, i \in V$. 

\subsection{Directed acyclic graphs and constraint-based methods} \label{sec_iiddags}
Let $X^i, i \in V$ be a set of random variables and $\mathcal G$ a directed acyclic graph (DAG) on $V$.
The joint distribution 
is said to be {\it Markov} with respect to the DAG $\mathcal G$ if 
each variable is independent of its non-descendants given its parents.
The distribution 
is {\it faithful} with respect to $\mathcal G$ if
all conditional independences are entailed by the Markov assumption.
Constraint-based methods \citep[e.g.][]{Spirtes2000} assume that the joint distribution is Markov, and faithful with respect to the true causal DAG.
They show how to exploit conditional independences for reconstructing the graph $\mathcal G$, e.g. using the PC algorithm; but the graph can only be recovered up to {\it Markov equivalence classes}. 
E.g., $X \rightarrow Y$ and $Y \rightarrow X$ cannot be distinguished. 

\subsection{Functional models and additive noise} \label{sec_iidfm}
Functional models \citep{Pearl2009} provide a different approach to the problem described above: 
We say 
$\prob^{(X^i), i \in V}$ satisfies a \emph{functional model} if for all $i \in V$ there exists a set of nodes $\PA[i]{} \subseteq X^{V \setminus \{i\}}$, a function $f_i$ and a noise variable $N^i$, such that we can write
$
X^i=f_i(\PA[i]{}, N^i)\,.
$
(For any subset $\B{A} \subset V$ we define $X_\B{A}:=\{X^j\,|\, j \in A\}$.
Additionally, we require $(N^i)_{i\in V}$ to be jointly independent and the graph obtained by drawing arrows from all elements of $\PA[i]{}$ to $X^i$ (for each $i \in V$) to be acyclic.
By restricting the function class one can identify the bivariate case: \citet{Shimizu2006} show that if $\prob^{(X,Y)}$ allows for $Y=a\cdot X+N_Y$ with $N_Y \independent X$ then $\prob^{(X,Y)}$ only allows for $X=b\cdot Y+N_X$ if $(X,N_Y)$ are jointly Gaussian ($\independent$ stands for statistical independence). 
This idea has led to the extensions of nonlinear additive functions $f(x, n) = g(x) + n$ \citep{Hoyer2008}, post-nonlinear additive functions $f(x, n) = h\big(g(x) + n\big)$ \citep{Zhang2009} and discrete functions \citep{Peters2011d}. \citet{Peters2011} show that identifiability in the bivariate case is enough for multiple variables. 
\citet{Mooij2009} 
provides practical ANM-based methods for more than two variables.
Sections~\ref{sec_timino} and \ref{sec_alg} apply these principles to time series.

\section{Causal inference on time series: existing methods} \label{sec:3}
For each $i$ from a finite $V$, let $\big(X^i_t\big)_{t \in \N}$ be a time series. $\X_t$ denotes the vector of time series values at time $t$.
We call the infinite graph that contains each variable $X^i_t$ as a node the {\it full time graph}. The {\it summary time graph} contains all $\#V$ components of the time series as vertices and an arrow between $X^i$ and $X^j$, $i \neq j$, if there is an arrow from $X^i_{t-k}$ to $X^j_t$ in the full time graph for some $k$. This work addresses the following

{\bf Problem:}
{\it Given a sample $(\X_1, \ldots, \X_T)$ of a multivariate time series, recover the true causal summary time graph. 
}



\subsection{Granger causality} \label{sec_granger}
Granger causality \citep{Granger1969} ({\it G-causality} for the remainder of the article)
does not require complicated statistics, it is easy to implement, and it is based on the following idea: $X^i$ does not Granger cause $X^j$ if including the past of $X^i$ does not help in predicting $X^j_t$ given the past of 
all all other time series $X^k$, $k \neq i$. In principle, ``all other'' means all other information in the world. In practice, one is limited to $X^k$, $k \in V$.
In order to translate the phrase ``does not help'' into the mathematical language we need to assume a multivariate time series model.
If the data follow the assumed model, e.g. the VAR model below, G-causality is sometimes interpreted as testing whether $X^i_{t-h}, h>0$ is independent of $X^j_t$ given $X^k_{t-h}$, $k\in V \setminus \{i\}, h>0$ 
\citep[see][and Section \ref{sec:anl}]{Florens1982, Eichler2011, Chu2008}.
\subsubsection{Linear Granger causality} \label{sec:lg}
Linear G-causality considers a VAR model:
$
\X_t=\sum_{\tau=1}^{p} \mathbf{A}(\tau) \X_{t-\tau} +\mathbf{\eps}_t\,,
$
where $\X_t$ and $\mathbf{\eps}_t$ are vectors and $\mathbf{A}(\tau)$ are matrices. 
For checking whether $X^i$ G-causes $X^j$ 
one fits a full VAR model $M_{\text{full}}$ to $\X_t$ and a VAR model $M_{\text{restr}}$ to $\X_t$ with the constraints $A_{\,\cdot \, i}(\tau)=0$ for all $1 \leq \tau \leq p$ that predicts $X^i_t$ without using $X^j$. Then one checks whether the reduction of the residual sum of squares (RSS) of $X^i_t$ is significant by using the following test statistic:
$T:=\frac{(RSS_{\text{restr}}-RSS_{\text{full}})/(p_{\text{full}}-p_{\text{restr}})}{RSS_{\text{full}}/(N-p_{\text{full}})}$,
where $p_{\text{full}}$ and $p_{\text{restr}}$ are the number of parameters in the respective models.
For the significance test we use $T \sim F_{p_{\text{full}}-p_{\text{restr}}, N-p_{\text{full}}}$. 

\subsubsection{Nonlinear Granger causality} \label{sec:nlg}
G-causality has been extended to nonlinear relationships, \citep[e.g.][]{Chen2004, Ancona2004}.
In this paper we focus on an extension for the bivariate case proposed by \citet{Bell1996}. It is based on generalized additive models (gams) \citep{Hastie1990}:
$X^i_t=\sum_{\tau=1}^{p} \sum_{j=1}^n f_{i,j,\tau}(X^j_{t-\tau}) +\eps^i_t$,
where $\mathbf{\eps}_t$ is a $\#V$ dimensional noise vector. In order to test whether $X^2$ G-causes $X^1$, for example with order $1$, two models are fit:
$X^1_t=g_1(X^1_{t-1})+\eps_t$ and $X^1_t=g_2(X^1_{t-1})+g_3(X^2_{t-1})+M_t$.
\citet{Bell1996} utilize the same $F$ statistic as above; this time $p_{\text{full}}$ and $p_{\text{restr}}$ are the estimated degrees of freedom of the corresponding models.
They refer to simulation studies by \citet{Hastie1990}.


\subsection{ANLTSM} \label{sec:anl}
Following \citet{Bell1996}, \citet{Chu2008} introduce additive nonlinear time series models (ANLTSM for short) for performing relaxed conditional independence tests: If including one variable, e.g. $X^1_{t-1}$, into a model for $X^2_t$ that already includes $X^2_{t-2}, X^2_{t-1},$ and $X^1_{t-2}$ does not improve the predictability of $X^2_t$, then $X^1_{t-1}$ is said to be independent of $X^2_t$ given $X^2_{t-2}, X^2_{t-1}, X^1_{t-2}$ (if the maximal time lag is $2$). \citet{Chu2008} propose a method based on constraint-based methods like FCI \citep{Spirtes2000} in order to infer the causal structure exploiting those conditional independence statements. The instantaneous effects are assumed to be linear and the confounders linear and instantaneous. Unfortunately, we did not find code for this method.

\subsection{TS-LiNGAM}
LiNGAM \citep{Shimizu2006} infers causal graphs for linear, non-Gaussian data. It has been extended to time series by \citet{Hyvarinen2008} (for short: TS-LiNGAM). It allows for instantaneous effects, all relationships are assumed to be linear. Hidden confounders and nonlinearities may lead to wrong results.

\subsection{Limitations of existing methods} \label{sec:lgc}
The approaches described above suffer from the following methodological problems:
(1) {\it Instantaneous effects:} The formulation of G-causality has the intrinsic problem that it cannot deal with instantaneous effects. E.g., when $X_t$ is causing $Y_t$, including any of the two time series helps for predicting the other. Thus G-causality infers $X \rightarrow Y$ and $Y \rightarrow X$.
ANLTSM and TS-LiNGAM only allow linear instantaneous effects. Theorem~\ref{thm:timino} shows that the causal summary time graph may still be identifiable when the instantaneous effects are linear and the variables are jointly Gaussian. TS-LiNGAM does not work in these situations.
(2) {\it Confounders:} G-causality might fail when there is a confounder between $X_{t}$ and $Y_{t+1}$, for example: 
The path between $X_t$ and $Y_{t+1}$ cannot be blocked by conditioning on any of the observed variables; G-causality infers $X \rightarrow Y$.
ANLTSM does not allow for nonlinear confounders or confounders with time structure and TS-LiNGAM may fail, too (Exp.~1).
(3) {\it Bad model assumptions:}
The methods share a similar problem: 
Performing general conditional independence tests is desirable, but not feasible, partially because the conditioning sets are too large \citep[e.g.][]{Bergsma2004}. Thus, the test is performed under a simple model, for example a linear one.
If the model assumption is violated, one may draw wrong conclusions without noticing (e.g. 
Exp.~3).
For TiMINo, that we define below, Lemma~\ref{lem:lmc} shows that after fitting and checking the model by testing for independent residuals, the difficult conditional independences have been checked implicitly.

Thus, a {\it model check} is a simple but effective improvement. Although G-causality for two time series can easily be augmented with a cross-correlation test, we do not see a straight-forward extension to the multivariate G-causality. Furthermore, testing for cross-correlation does not always suffice (see Section~\ref{sec_full}).

\section{Functional models for time series: TiMINo} \label{sec_timino}
We define TiMINo, a model class including the models described above and prove its identifiability.  
\begin{definition} \label{def:timino}
Consider a time series $\X_t=(X^i_t)_{i \in V}$, such that the finite dimensional distributions are absolutely continuous with respect to a product measure (i.e. there is a pdf or a pmf). We say the time series satisfies a \emph{TiMINo} if there is a $p>0$ and if $\forall i \in V$ there are sets $\PA[i]{0} \subseteq X^{V\setminus\{i\}}, \PA[i]{k} \subseteq X^V$, s.t. $\forall t$
\begin{align} \label{anm}
X^i_t = f_{i}\big((\PA[i]{p})_{t-p}, \ldots, (\PA[i]{1})_{t-1}, (\PA[i]{0})_{t}, \eps^i_t\big)\,,
\end{align}
with  
$\eps^i_t$ (jointly) independent and for each $i$, $\eps^i_t$ identically distributed in $t$. The corresponding full time graph is obtained by drawing arrows from any node that appears in the right-hand side of \eqref{anm} to $X_t^i$. We require the full time graph to be acyclic.
\end{definition}

Below we assume that equations \eqref{anm} follow an identifiable functional model class (IFMOC), \citet{Peters2011} give a precise definition. Basically, it means that 
(I) {\it causal minimality} holds, a weak form of faithfulness that assumes a statistical dependence between cause and effect given all other parents \citep[][]{Spirtes2000}. 
And (II), all $f_i$ come from a function class (e.g. additive noise) that is small enough to make the bivariate case identifiable (Section \ref{sec_iidfm}) if we exclude certain function-input-noise combinations like linear-Gaussian-Gaussian.
The proof of the following theoretical result can be found in the appendix.
\begin{Theorem} \label{thm:timino}
Suppose that $\X_t$ can be represented as a TiMINo with $\PA[]{}(X^i_t)=\bigcup_{k=0}^p (\PA[i]{k})_{t-k}$ being the direct causes of $X^i_t$ and that one of the following holds: \vspace{-0.2cm}
\begin{itemize}
\item[(i)] Equations \eqref{anm} come from an IFMOC.  \vspace{-0.2cm}
\item[(ii)] Each component of the time series exhibits a time structure (i.e. $\PA[]{}(X^i_t)$ contains at least one $X^i_{t-k}$), the joint distribution is faithful with respect to the full time graph, and the summary time graph is acyclic.  \vspace{-0.2cm}
\end{itemize}
Then the full time graph can be recovered from the joint distribution. In particular, the true causal summary time graph is identifiable. (Note that neither of the two conditions implies the other.)
\end{Theorem}
Regarding (i): Many choices of a function class are possible \citep{Peters2011}. In practice, however, one still needs to fit those functions $f_i$ from the data, which means for additive noise that
estimating $\mean[X^i_t|\X_{t-p}, \ldots, \X_{t-1}]$ should be feasible. Different results show that stationarity and/or $\alpha$ mixing, or geometric ergodicity are required \citep[e.g.][]{Chu2008}.
In this work we consider VAR fitting: $f_i(p_1, \ldots, p_r, n)=a_{i,1}\cdot p_1+ \ldots + a_{i,r} \cdot p_r + n$, gam regression: $f_i(p_1, \ldots, p_r, n)=f_{i,1}(p_1)+ \ldots + f_{i,r}(p_r)+ n$ \citep[e.g.][]{Bell1996}, and GP regression: $f_i(p_1, \ldots, p_r, n)=f_i(p_1, \ldots, p_r) + n$.
Note that linear functions lead to the model of \citet{Hyvarinen2008} as a special case.\\
Regarding (ii): This condition nicely shows how the time structure does not only make the causal inference problem harder (the iid assumption is dropped), but also easier. In the iid case, for example, the true graph is not identifiable if all components are jointly Gaussian and the relationships are linear; with time structure it is. (TS-LiNGAM would fail, though.)

\section{A practical method: TiMINo causality} \label{sec_alg}
The algorithm for TiMINo causality is based on the theoretical finding in Theorem \ref{thm:timino}. It takes the time series data as input and outputs either a DAG that estimates the summary time graph or remains undecided. In principle, it tries to fit a TiMINo model to the data and outputs the corresponding graph. If no model with independent residuals is found, it outputs ``I do not know''.
For a time series with many components, this gets intractable. In Section \ref{sec_exp}, we concentrate on time series without feedback loops, where we can exploit a more efficient method:



\subsection{Full causal discovery} \label{sec_full}
For additive noise models (ANMs) without time structure, \citet{Mooij2009} propose a procedure that recovers the structure without enumerating all possible DAGs. This procedure can be modified to be of use for time series (Algorithm~1). 
As reported by \citet{Mooij2009}, the time complexity is $\mathcal O(d^2)$, where $d$ is the number of time series, regarding fitting models and independence testing as atomic operations. To get the full time complexity, $\mathcal O(d^2)$ has to be multiplied by the sum of the complexity of the regression method and the independence test, both chosen by the user.
%
\begin{algorithm}[h]
\caption{TiMINo causality}
\begin{algorithmic}[1]
   \STATE {\bfseries Input:} Samples from a $d$-dimensional time series of length $T$: $(\X_1, \ldots, \X_T)$, maximal order $p$\vspace{0.0cm}
   \STATE $S:=(1, \ldots, d)$ 
   \REPEAT
   \FOR{$k$ in S}
   \STATE Fit TiMINo for $X^k_t$ using $X^k_{t-p}, \ldots,$ $X^k_{t-1}, X^i_{t-p}, \ldots, X^i_{t-1}, X^i_t$ for $i \in S \setminus \{k\}$
   \STATE Test if residuals are indep. of $X^i, i \in S$.  
   \ENDFOR
   \STATE Choose $k^*$ to be the $k$ with the weakest dependence. (If there is no $k$ with independence, break and output: ``I do not know - bad model fit'').
   \STATE $S:=S \setminus \{k^*\}$
   \STATE $\mathrm{pa}(k^*):=S$
   \UNTIL{length(S)=1}
   \STATE For all $k$ remove all unnecessary parents.\vspace{0.0cm}
   \STATE {\bfseries Output:} $(\mathrm{pa(1)}, \ldots, \mathrm{pa(d)})$
\end{algorithmic}
\end{algorithm}

Depending on the assumed model class, TiMINo causality has to be provided with a fitting method. Here, we chose $\mathtt{ar}$, $\mathtt{gam}$ and $\mathtt{gptk}$ in R (\url{http://www.r-project.org/}) for linear models, generalized additive models, and GP regression, We call the methods TiMINo-linear, TiMINo-gam and TiMINo-GP, respectively. For the first two AIC determines the order of the process. 
All fitting methods are used in a ``standard way''. For $\mathtt{gam}$ we used the built-in nonparametric smoothing splines.
For the GP we used zero mean, squared exponential covariance function and Gaussian Likelihood. The hyper-parameters are automatically chosen by marginal likelihood optimization.

To test for independence between a residual time series $N^k_t$ and another time series $X_t^i, i \in S$, we shift the latter time series up to the maximal order $\pm p$ (but at least up to $\pm 4$); for each of those combinations we perform HSIC \citep{GreFukTeoSonetal08}, an independence test for iid data. 
One could also use a test based on cross-correlation that can be derived from Thm 11.2.3. in \citep[][]{brockwelldavis}. 
This is related to what is done in transfer function modeling \citep[e.g. \S $13.1$ in][]{brockwelldavis}, which is restricted to two time series and linear functions. But testing for cross-correlation is often not enough: if no time structure is present (iid data), it is obvious that correlation tests are most often insufficient. Also, experiments 1 and 5 describe situations, in which cross-correlations fail.
To reduce the running time, however, one can use cross-correlation to determine the graph structure and use HSIC as a final model check. For HSIC we used a Gaussian kernel; as in \citep{GreFukTeoSonetal08}, the bandwidth is chosen to be the median distance of the input data. This is a heuristic but well-established choice.

Note that any other fitting method and independence test can be used as well. Although they work well in practice, we do not claim that our choices are optimal. 

\subsection{Partial causal discovery} \label{sec_partial}
Let $\X_t$ ``almost'' satisfy a TiMINo model, that is some time series are unobserved or some functional relationships are not included in the model. We expect that the full discovery method remains undecided. One can modify the method such that it tries to discover parts of the causal graph:
Whenever no $k$ with independent residuals is found in line 8 of Algorithm~1 one subtracts a subset $S_0$ from the current version of $S$ (first subtract one element, then any combination of two etc.) and repeat. If the method is able to fit a TiMINo model using only the remaining set $S \setminus S_0$, output this solution and $S_0$, which has been excluded. Since there are $2^{\#S}$ subsets, this is only feasible for small $S$ 
(see Exp. 6). This method may also be useful for the iid case; its theoretical properties remain to be investigated.

\subsection{Weaknesses} \label{sec:wea}
(i) In principle, it may happen that the model assumption are violated, but one can nevertheless fit a model in the wrong direction (that is why we wrote ``remaining {\it mostly} undecided''). This requires an ``unnatural'' fine tuning of the functions and \citet{Steudel2010} argue that in the case of causal sufficiency it cannot occur if one believes in the ``independence'' of cause and causal mechanism. Also, (i) is relevant only when there are time series without time structure or the data are non-faithful (see Theorem~\ref{thm:timino}). We do not provide a precise analysis of the case with confounders, but analyze this situation empirically in Experiment~1.
(ii) The null hypothesis of the independence test represents independence, although the scientific discovery of a causal relationship should rather be the alternative hypothesis.
This fact may lead to wrong causal conclusions (instead of ``I do not know'') on small data sets since we cannot reject independence for the wrong direction.
This effect is strengthened by the Bonferroni correction of the HSIC based independence test. This may require modifications, when the number of time series is very high. 
It is thus useful to develop heuristics for ``minimal'' sample sizes.
(iii) For large sample sizes, even smallest differences between the true data generating process and the model may lead to rejected independence tests \citep[discussed by][]{Peters2011d}.

\section{Experiments} \label{sec_exp}
Code is available in the suppl. mat. and will be online.
\subsection{Artificial Data}

We always included instantaneous effects,  
fitted models up to order $p=2$ or $p=6$ and set $\alpha=0.05$.

{\bf Experiment 1: Confounder with time lag.}
We simulate $100$ data sets (length $1000$) from
$Z_t=a\cdot Z_{t-1}+\eps_{Z,t},
X_t=0.6\cdot X_{t-1}+0.5\cdot Z_{t-1}+\eps_{X,t},
Y_t=0.6\cdot Y_{t-1}+0.5\cdot Z_{t-2}+\eps_{Y,t},$
with $a$ between $0$ and $0.95$ and $\eps_{\cdot, t} \sim 0.4 \cdot \mathcal N(0,1)^3$. Here, $Z$ is a hidden common cause for $X$ and $Y$.
For all $a$, $X_t$ contains information about $Z_{t-1}$ and $Y_{t+1}$ (see Figure~\ref{fig:tg2}); G-causality and TS-LiNGAM wrongly infer $X \rightarrow Y$. For large $a$, $Y_t$ contains additional information about $X_{t+1}$, 
which leads to the wrong arrow $Y \rightarrow X$. TiMINo causality does not decide for any $a$. 
The nonlinear methods perform very similar (not shown). Note that for $a=0$, a cross-correlation test is not enough to reject $X \rightarrow Y$. Further, all methods fail for $a=0$ and Gaussian noise.

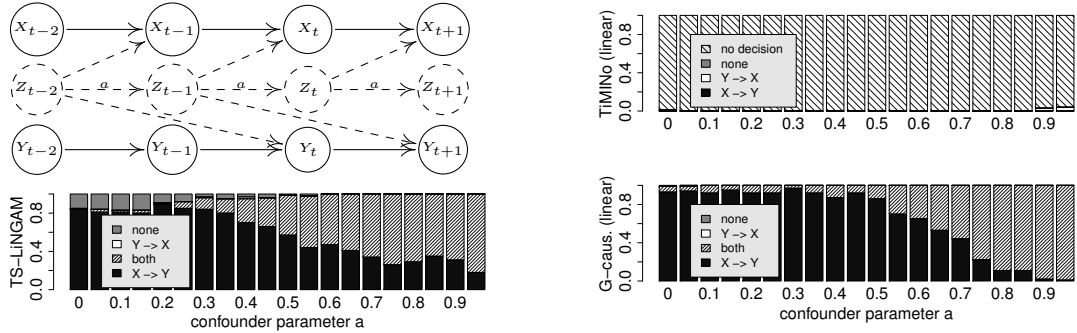
\begin{figure*}[t]
\begin{center}
\begin{minipage}{0.40\textwidth}
\begin{tikzpicture}[xscale=1.2, yscale=0.8, line width=0.4pt, inner sep=0.2mm, shorten >=1pt, shorten <=1pt]
  \tiny
  \draw (0,2) node(x0) [circle, draw] {$X_{t-2}$};
  \draw (1.5,2) node(x1) [circle, draw] {$X_{t-1}$};
  \draw (3,2) node(x2) [circle, draw] {$\;\;X_{t}\;\,$};
  \draw (4.5,2) node(x3) [circle, draw] {$X_{t+1}$};
  \draw (0,0) node(y0) [circle, draw] {$Y_{t-2}$};
  \draw (1.5,0) node(y1) [circle, draw] {$Y_{t-1}$};
  \draw (3,0) node(y2) [circle, draw] {$\;\;Y_{t}\;\,$};
  \draw (4.5,0) node(y3) [circle, draw] {$Y_{t+1}$};
  \draw (0,1) node(z0) [circle, draw, dashed] {$Z_{t-2}$};
  \draw (1.5,1) node(z1) [circle, draw, dashed] {$Z_{t-1}$};
  \draw (3,1) node(z2) [circle, draw, dashed] {$\;\;Z_{t}\;\,$};
  \draw (4.5,1) node(z3) [circle, draw, dashed] {$Z_{t+1}$};
  \draw[-arcsq] (x0) -- (x1);
  \draw[-arcsq] (x1) -- (x2);
  \draw[-arcsq] (x2) -- (x3);
  \draw[-arcsq] (y0) -- (y1);
  \draw[-arcsq] (y1) -- (y2);
  \draw[-arcsq] (y2) -- (y3);
  \draw[-arcsq, dashed] (z0) -- (x1);
  \draw[-arcsq, dashed] (z1) -- (x2);
  \draw[-arcsq, dashed] (z2) -- (x3);
  \draw[-arcsq, dashed] (z0) -- (y2);
  \draw[-arcsq, dashed] (z1) -- (y3);
  \draw[-arcsq, dashed] (z0) to node [above] {$a$} (z1);
  \draw[-arcsq, dashed] (z1) to node [above] {$a$} (z2);
  \draw[-arcsq, dashed] (z2) to node [above] {$a$} (z3);
\end{tikzpicture}
\vspace{0.2cm}\\
\includegraphics[width=0.99\linewidth]{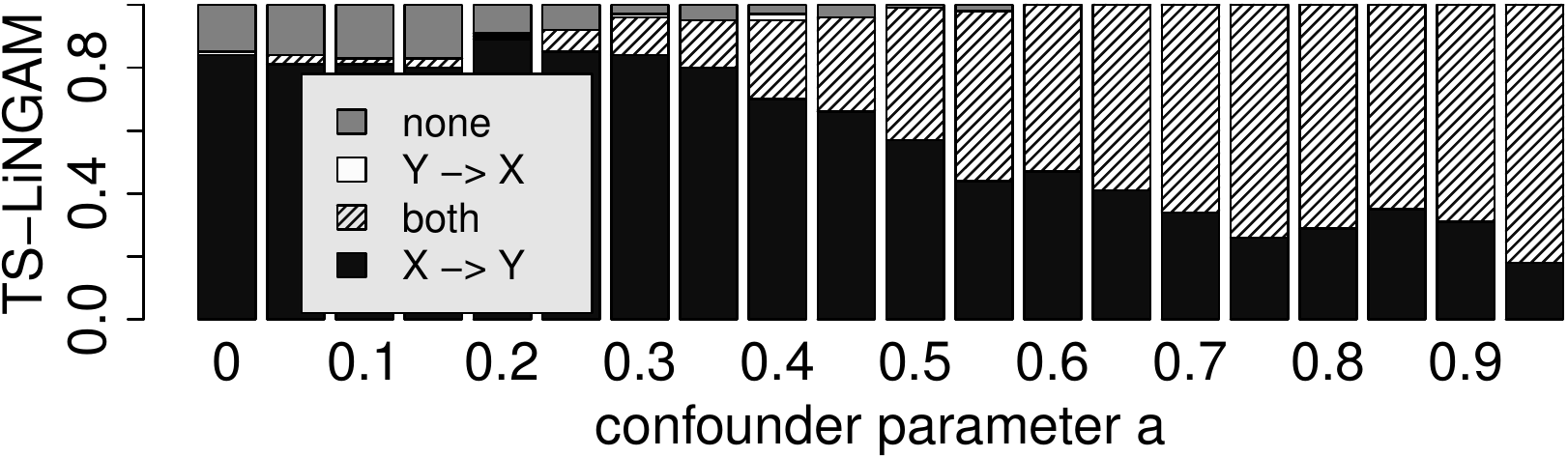}
\end{minipage}
\hspace{1.2cm}
\begin{minipage}{0.40\textwidth}
\includegraphics[width=0.99\linewidth]{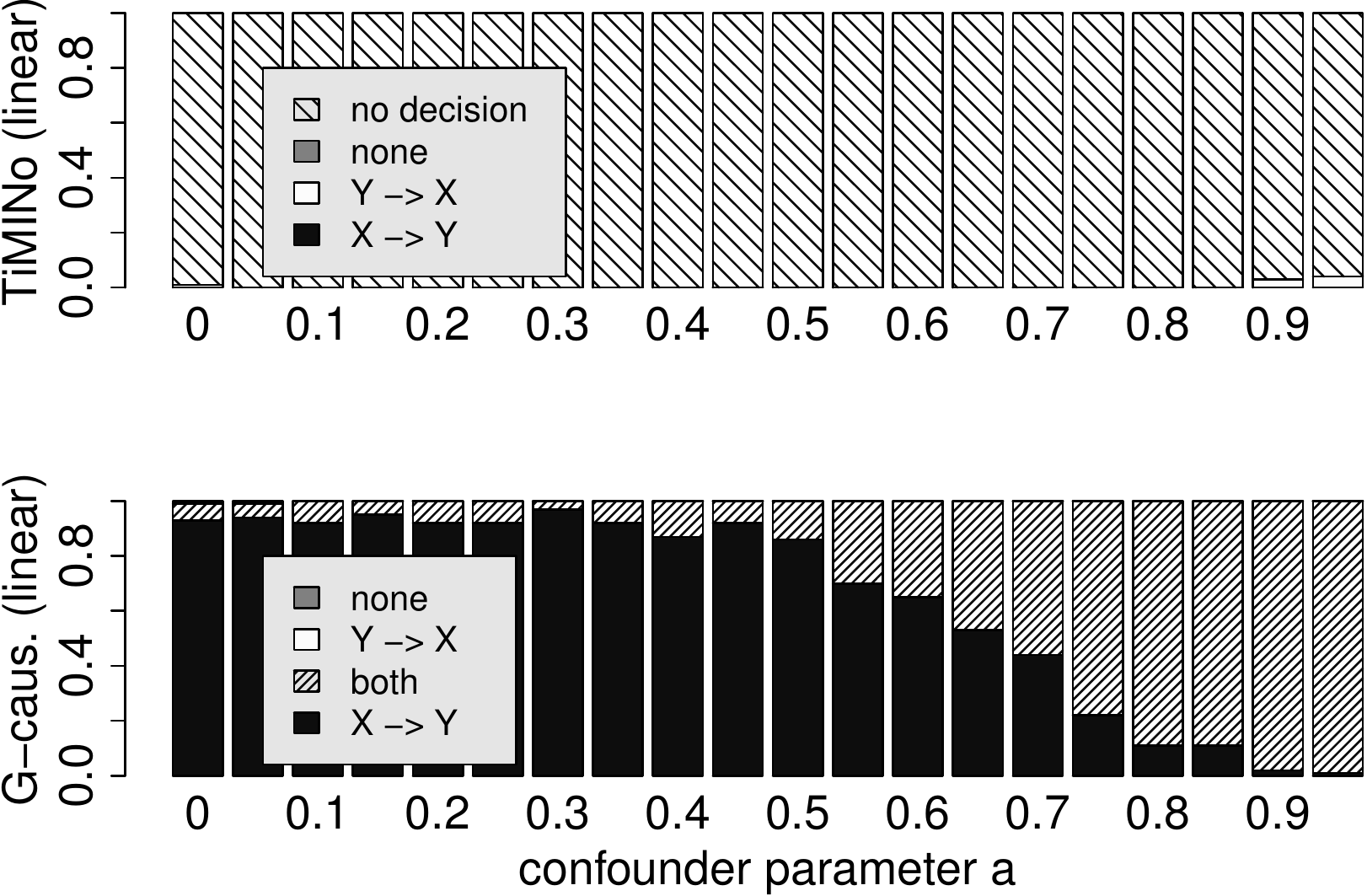}
\end{minipage}
\end{center}
\caption{Exp.1: Part of the causal full time graph with hidden common cause $Z$ (top left). TiMINo causality does not decide (top right), whereas G-causality and TS-LiNGAM wrongly infer causal connections between $X$ and $Y$ (bottom).}
\label{fig:tg2}
\end{figure*}



{\bf Experiment 2: Linear, Gaussian with instantaneous effects.}
We sample $100$ data sets (length $2000$) from 
$X_t=A_1\cdot X_{t-1}+\eps_{X,t},
W_t=A_2\cdot W_{t-1}+A_3\cdot X_{t}+\eps_{W,t},
Y_t=A_4\cdot Y_{t-1}+A_5\cdot W_{t-1}+\eps_{Y,t},
Z_t=A_6\cdot Z_{t-1}+A_7\cdot W_t+A_8\cdot Y_{t-1}+\eps_{Z,t}
$
and $\eps_{\cdot, t} \sim 0.4 \cdot \mathcal N(0,1)$ and $A_i$ iid from $\mathcal{U}([-0.8,-0.2]\cup [0.2,0.8])$. We regard the graph containing $X \rightarrow W \rightarrow Y \rightarrow Z$ and $W \rightarrow Z$ as correct. TS-LiNGAM and G-causality are not able to recover the true structure (see Table \ref{tab:multiv_lin}).
\begin{table}[h]
\caption{Exp.2: 
Gaussian data and linear instantaneous effects: only TiMINo mostly discovers the correct DAG.}
\label{tab:multiv_lin}
\begin{center}
\small
\begin{tabular}{r||c|c|c}
DAG&lin. Granger&TiMINo-lin&TS-LiNGAM\\\hline \hline
correct&$13\%$&$83\%$&$19\%$\\\hline
wrong&$87\%$&$7\%$&$81\%$\\\hline
no dec.&$0\%$&$10\%$&$0\%$
\end{tabular}
\normalsize
\end{center}
\end{table}
%
%

{\bf Experiment 3: Nonlinear, non-Gaussian without instantaneous effects.}
We simulate $100$ data sets (length $500$) from
$X_t=0.8 X_{t-1} + 0.3 \eps_{X,t},
Y_t=0.4 Y_{t-1} + (X_{t-1}-1)^2 +0.3\eps_{Y,t},
Z_t=0.4 Z_{t-1} + 0.5 \cos(Y_{t-1}) + \sin(Y_{t-1}) +0.3\eps_{Z,t},$
with $\eps_{\cdot, t} \sim \mathcal U([-0.5,0.5])$ (similar results for other noise distributions, e.g. exponential).
Thus, $X \rightarrow Y \rightarrow Z$ is the ground truth. Nonlinear G-causality fails since the implementation is only pairwise and it thus always infers an effect from $X$ to $Z$. Linear G-causality cannot remove the nonlinear effect from $X_{t-2}$ to $Z_t$ by using $Y_{t-1}$ and gives many wrong answers. Also TiMINo-linear assumes a wrong model, but does not make any decision. TiMINo-gam and TiMINo-GP work well on this data set (Table~\ref{tab:multiv_nonlin}). This specific choice of parameters show that a significant difference in performance is possible. For other parameters (e.g. less impact of the nonlinearity), G-causality and TS-LiNGAM still assume a wrong model but make fewer mistakes.

\begin{table*}[t]
\caption{Exp.3: Since the data are nonlinear, linear G-causality and TS-LiNGAM give wrong answers, TiMINo-lin does not decide. Nonlinear G-causality fails because it analyzes the causal structure between pairs of time series.}
\label{tab:multiv_nonlin}
\begin{center}
\small
\begin{tabular}{r||c|c|c|c|c|c}
DAG&Granger (lin)&Granger (nonlin)&TiMINo (lin)&TiMINo (gam)&TiMINo (GP)&TS-LiNGAM\\\hline \hline
correct&$69\%$&$  0\%$&$  0\%$&$95\%$&$94\%$&$12\%$\\\hline
wrong  &$31\%$&$100\%$&$  0\%$&$1\%$&$1\%$&$88\%$\\\hline
no dec.&$0\%$&$   0\%$&$100\%$&$ 4\%$&$5\%$&$ 0\%$
\end{tabular}
\end{center}
\end{table*}
\normalsize

{\bf Experiment 4: Non-additive interaction.}
We simulate $100$ data sets with different lengths from
$X_t=0.2 \cdot X_{t-1} +0.9\eps_{X,t},
Y_t=-0.5 + \exp(-(X_{t-1} + X_{t-2})^2)+0.1\eps_{Y,t}$,
with $\eps_{\cdot, t} \sim \mathcal N(0,1)$.
Figure~2 shows that TiMINo-linear and TiMINo-gam remain mainly undecided, whereas TiMINo-GP performs well. For small sample sizes, one observes two effects: GP regression does not obtain accurate estimates for the residuals, these estimates are not independent and thus TiMINo-GP remains more often undecided. Also, TiMINo-gam makes more correct answers than one would expect due to more type II errors.
Linear G-causality and TS-LiNGAM give more than $90\%$ incorrect answers, but non-linear G-causality is most often correct (not shown). Bad model assumptions do not {\it always} lead to incorrect causal conclusions.
\begin{figure}[h]
\centering
\includegraphics[width=0.63\linewidth]{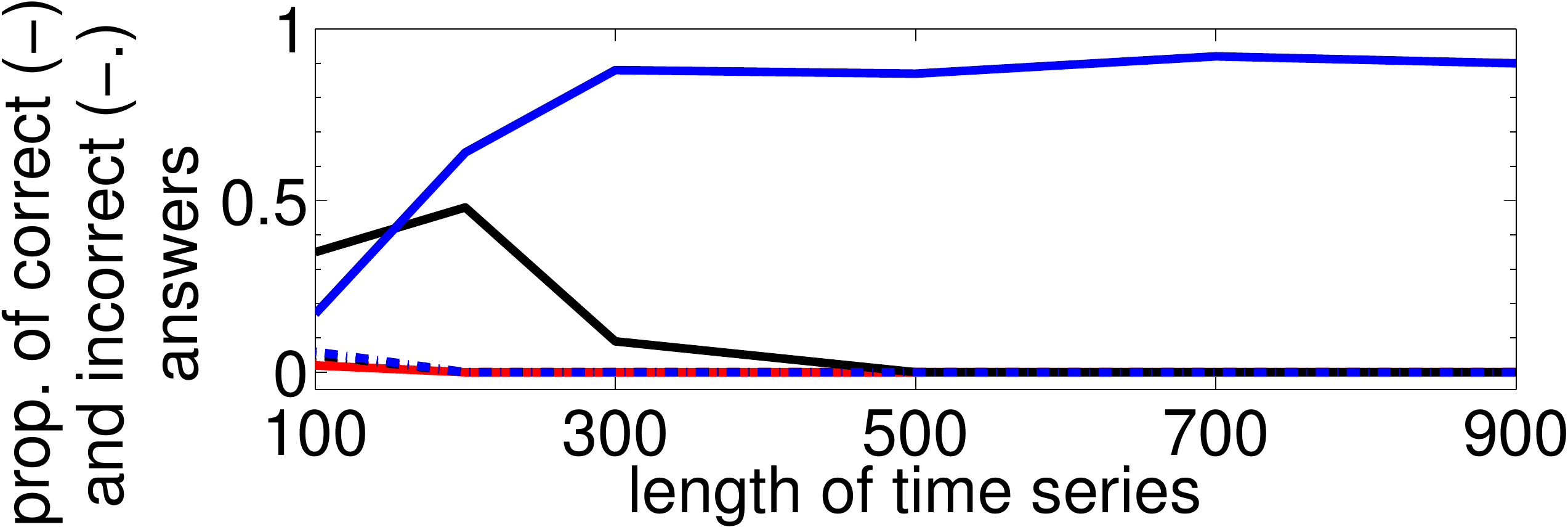}
\caption{Exp.4: TiMINo-GP (blue) works reliably for long time series. TiMINo-linear (red) and TiMINo-gam (black) mostly remain undecided.}
\label{fig:fig1}
\end{figure}


{\bf Experiment 5: Non-linear Dependence of Residuals.}
In Experiment~1, TiMINo equipped with a cross-correlation inferred a causal edge, although there were none. The opposite is also possible:
$X_t=-0.5\cdot X_{t-1}+\eps_{X,t},
Y_t=-0.5\cdot Y_{t-1}+ X_{t-1}^2 +\eps_{Y,t}$
and $\eps_{\cdot, t} \sim 0.4 \cdot \mathcal N(0,1)$ (length $1000$). TiMINo-gam with cross-correlation infers no causal link between $X$ and $Y$, whereas TiMINo-gam with HSIC correctly identifies $X \rightarrow Y$.

{\bf Experiment 6: Partial Causal Discovery.}
We sample $100$ data sets (length $600$) from 
$X_t=0.5\cdot X_{t-1}+\eps_{X,t},
B_t=0.5\cdot B_{t-1}+\eps_{B,t},
A_t=0.5\cdot A_{t-1}+0.5 \cdot B_{t-1}+\eps_{A,t},
Y_t=0.5\cdot Y_{t-1}-0.9 \cdot X_{t-1}+0.8\cdot B_{t-1}+\eps_{Y,t},
W_t=0.5\cdot W_{t-1}+0.8 \cdot X_{t-1}+\eps_{W,t}$
and $\eps_{\cdot, t} \sim 0.4 \cdot \mathcal U([-0.5,0.5])$.
Let $X_t$ be latent. 
The standard method 
finds $A_t$ as a ``sink time series'' and halts in iteration two (line $8$ in Algorithm~1
). Instead of outputting ``I do not know'',
the partial discovery method described in Section~\ref{sec_partial} is able to correctly infer this DAG (see Figure~\ref{fig:partial}) in $82\%$ of the cases ($18\%$ wrong answers). G-causality and TS-LiNGAM give only wrong answers.

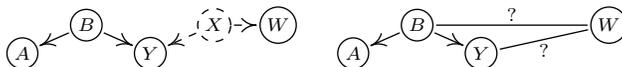
\begin{figure}[h]
\begin{center}
\begin{tikzpicture}[xscale=.85, yscale=0.75, line width=0.5pt, inner sep=0.5mm, shorten >=1pt, shorten <=1pt]
  \scriptsize
  \draw (0,0) node(a) [circle, draw] {$A$};
  \draw (1,0.5) node(b) [circle, draw] {$B$};
  \draw (2,0) node(y) [circle, draw] {$Y$};
  \draw (3,0.5) node(x) [circle, draw, dashed] {$X$};
  \draw (4,0.5) node(w) [circle, draw] {$W$};
  \draw[-arcsq] (b) -- (a);
  \draw[-arcsq] (b) -- (y);
  \draw[-arcsq, dashed] (x) -- (y);
  \draw[-arcsq, dashed] (x) -- (w);
\end{tikzpicture}
\hspace{0.3cm}
\begin{tikzpicture}[xscale=.85, yscale=0.75,line width=0.5pt, inner sep=0.5mm, shorten >=1pt, shorten <=1pt]
  \scriptsize  
  \draw (0,0) node(a) [circle, draw] {$A$};
  \draw (1,0.5) node(b) [circle, draw] {$B$};
  \draw (2,0) node(y) [circle, draw] {$Y$};
  \draw (4,0.5) node(w) [circle, draw] {$W$};
  \draw[-arcsq] (b) -- (a);
  \draw[-arcsq] (b) -- (y);
  \draw[bend right=0] (y) to node [below] {?} (w);
  \draw (b) to node [above] {?} (w);
\end{tikzpicture}
\end{center}
\caption{Exp.6: The true causal summary time graph (left) cannot be recovered because $X_t$ is unobserved. TiMINo gives a partial result (right).}
\label{fig:partial}
\end{figure}

\subsection{Real Data}
We fitted up to order $6$ and included instantaneous effects. For TiMINo, ``correct'' means that TiMINo-gam makes the correct decision and TiMINo-linear is correct or undecided. TiMINo-GP always remains undecided because there are too few data points to fit such a general model. Again, $\alpha$ is set to $0.05$. 

{\bf Experiment 7: Gas Furnace.}
\citep[][length $296$]{Box2008},
$X_t$ describes the input gas rate and $Y_t$ the output $\text{CO}_2$. We regard $X \rightarrow Y$ as being true. TS-LiNGAM, G-causality, TiMINo-lin and TiMINo-gam correctly infer $X \rightarrow Y$. Disregarding time information leads to a wrong causal conclusion: The method described by \citet{Hoyer2008} leads to a $p$-value of $4.8\%$ in the correct and $9.1\%$ in the false direction.

{\bf Experiment 8: Old Faithful.}
\citep[][length $194$]{Azzalini1990} $X_t$ contains the duration of an eruption and
$Y_t$ the time interval to the next eruption of the Old Faithful geyser.
We regard $X \rightarrow Y$ as the ground truth. Although the time intervals $[t, t+1]$ do not have the same length for all $t$, we model the data as two time series. TS-LiNGAM and TiMINo give correct answers, whereas linear G-causality  
infers $X \rightarrow Y$, and nonlinear G-causality infers $Y \rightarrow X$.

{\bf Experiment 9: Temperature.} (available at \url{https://webdav.tuebingen.mpg.de/cause-effect/}, length $16382$)
$X_t$ are indoor and $Y_t$ outdoor measurements (recorded every $5$ minutes), we expect that there is a causal link $Y \rightarrow X$. TS-LiNGAM wrongly infers $X \rightarrow Y$ and both G-causality methods infer a bidirected arrow. TiMINo remains undecided. Maybe, the data are causal insufficient: time may confound outdoor temperature and the usage of heating, the latter is a direct cause for indoor temperature. Also, $Y$ may cause heating. Such a model does not allow for a TiMINo from $Y$ to $X$.

{\bf Experiment 10: Abalone (no time structure).}
The abalone data set \citep{Asuncion} contains (among others that lead to similar results) age $X_t$ and diameter $Y_t$ of a certain shell fish. If we model $1000$ randomly chosen samples as time series, G-causality (both linear and nonlinear) infers no causal relation as expected. TS-LiNGAM wrongly infers $Y \rightarrow X$, which is probably due to the nonlinear relationship. TiMINo gives the correct result.

{\bf Experiment 11: Diary (confounder).}
We consider $10$ years of weekly prices for butter $X_t$ and cheddar cheese $Y_t$ \citep[][length $522$]{data_dairy}. They are strongly correlated, but we expect this correlation to be due to the (hidden) milk price $M_t$: $X \leftarrow M \rightarrow Y$. TiMINo does not decide, whereas TS-LiNGAM and G-causality 
wrongly infer $X \rightarrow Y$. This may be due to different time lags of the confounder (cheese has longer storing and maturing times than butter).

The phase slope index \citep{Nolte2008}
performed well only in Exp.~6, in all other experiments it either gave wrong results or did not decide.
Due to space constraints we omit details about this method. 



\section{Conclusions and Future Work} \label{sec_conclusions}
This paper shows how causal inference benefits from the framework of functional models.
TiMINo causality can be seen as an extension of methods from the iid case, but the benefits compared to other time series methods are substantial and important: 
It comes with an identifiability that is more general than existing results and lead to a practical algorithm that allows for the ability to make no decision instead of a wrong one.
TiMINo is applicable to multivariate, linear, nonlinear and instantaneous interactions and can also discover partial structures. On the data sets considered it outperforms existing methods.

We think the following investigations would be worthwhile: 
(1) Applying more complex models (like heteroscedastic models) and preprocessing the data (removing trends, periodicities, etc.) may decrease the number of cases where TiMINo causality is undecided. 
(2) Checking for autocorrelations in the residuals is another possible model check and not included yet.
(3) In the case of non-instantaneous feedback loops, one should find a method to fit the model structure that is faster than brute-force search.
(4)
Although we report promising results, an extensive evaluation of this method on even more real data sets is necessary. This lies beyond the scope of the present conference paper.

\section{Appendix}
\begin{Lemma} \label{lem:lmc}
If $\X_t=(X^i_t)_{i \in V}$ satisfy a TiMINo model, each variable $X^i_t$ is conditionally independent of each of its non-descendants given its parents.
\end{Lemma}
\begin{proof}
With 
$
\mathcal S:=\PA[]{}(X^i_t)=\bigcup_{k=0}^p (\PA[i]{k})_{t-k}
$ 
and equation \eqref{anm} we get 
$
X^i_t |_{\mathcal S=s} = f_{i}(s, \eps^i_t)
$
for an $s$ with $p(s)>0$. Any non-descendant of $X^i_t$ can be written as a function of all noise variables from its ancestors and is therefore independent of $X^i_t$ given $\C{S}=s$.
For this proof it is crucial that we consider time series for $t \in \N$. We believe that a similar statement holds for $t \in \Z$, which only introduces technical difficulties.  
\end{proof}
\begin{proof}[of Theorem \ref{thm:timino}]
Suppose that $\X_t$ allows two different representations of TiMINo that lead to two different full time graphs $\mathcal G$ and $\mathcal G'$.
(i) First we assume that $\mathcal G$ and $\mathcal G'$ do not differ in the instantaneous effects: $\PA[i]{0} (\text{in }\mathcal G)=\PA[i]{0}(\text{in }\mathcal G')\; \forall i$. Without loss of generality, there is some $k>0$ and an edge $X^1_{t-k} \rightarrow X^2_t$, say, that is in $\mathcal G$ but not in $\mathcal G'$. 
From $\mathcal G'$ and Lemma \ref{lem:lmc} we have that
$
X^1_{t-k} \independent X^2_t\,|\,\mathcal S\,,
$
where $\mathcal S=(\{X^i_{t-l}, 1 \leq l \leq p, i \in V\} \cup \B{ND}_t)\setminus \{X^1_{t-k}, X^2_t\}$, and $\B{ND}_t$ are all $X^i_t$ that are non-descendants (wrt instantaneous effects) of $X^2_t$.
Applied to $\mathcal G$, 
causal minimality
leads to a contradiction:
$
X^1_{t-k} \notindependent X^2_t\,|\,\mathcal S\,.
$
Now we suppose $\mathcal G$ and $\mathcal G'$ differ in the instantaneous effects. This time we choose $\mathcal S=\{X^i_{t-l}, 1 \leq l \leq p, i \in V\}$. Then for each $s$ and $i$ we have:
$
X^i_t |_{\mathcal S=s}=f_i(s, (\tPA[i]{0})_t) 
$, 
where $\tPA[i]{0}$ are all instantaneous parents of $X^i_t$ conditioned on $\mathcal S=s$. All $X^i_t |_{\mathcal S=s}$ with the instantaneous effects describe two different structures of an IFMOC. This contradicts the identifiability results by \citet{Peters2011}.
(ii) Because of Lemma \ref{lem:lmc} and faithfulness $\mathcal G$ and $\mathcal G'$ only differs in the instantaneous effects. But each instantaneous arrow $X_t^i \rightarrow X_t^j$ forms a $v$-structure together with $X_{t-k}^j \rightarrow X_t^j$; the latter exists because of the time structure and $X_{t-k}^j$ cannot be connected with $X_t^i$ since this introduces a cycle in the summary time graph. 
\end{proof}


\bibliographystyle{plainnat}
\bibliography{bibliography}

\end{document}